\documentclass[runningheads,a4paper]{llncs}

\usepackage{amssymb}
\usepackage{amsmath}
\usepackage{bbm}
\usepackage{graphicx}

\usepackage{url}
\urldef{\mailsc}\path|gilles.blanchard@math.uni-potsdam.de|
\urldef{\mailsa}\path|ilya@tuebingen.mpg.de|
\urldef{\mailsb}\path|nikita.zhivotovskiy@phystech.edu|    
\newcommand{\keywords}[1]{\par\addvspace\baselineskip
\noindent\keywordname\enspace\ignorespaces#1}

\newcommand{\R}{\mathbb{R}}
\newcommand{\PRC}{\hat{Q}}
\newcommand{\Radc}{\hat{{R}}}
\newcommand{\Radctd}{\hat{R}^{td}}

\newcommand{\X}{\mathcal{X}}
\newcommand{\E}{\mathop{\mathbb{E}}}
\newcommand{\C}{\mathcal{C}}
\newcommand{\Z}{\mathcal{Z}}
\renewcommand{\S}{\mathcal{S}}

\newcommand{\fa}{\bar{f}}
\newcommand{\err}{\mathrm{err}}
\newcommand{\vepsilon}{\vec \epsilon}
\newcommand{\veta}{\vec \eta}
\newcommand{\Y}{\mathcal{Y}}
\newcommand{\card}{\mathrm{card}}
\newcommand{\Hyp}{\mathcal{H}}

\newcommand{\Prob}{\mathbb{P}}

\begin{document}

\mainmatter  

\title{Permutational Rademacher Complexity}
\subtitle{A New Complexity Measure for Transductive Learning}

\titlerunning{Permutational Rademacher Complexity}

\author{
Ilya Tolstikhin\inst{1}
\and Nikita Zhivotovskiy\inst{2}\inst{3}
\and Gilles Blanchard\inst{4}
}

%

\institute{
Max-Planck-Institute for Intelligent Systems, T\"ubingen, Germany
\mailsa
\and
Moscow Institute of Physics and Technology, Moscow, Russia
\and
Institute for Information Transmission Problems, Moscow, Russia
\mailsb
\and
Department of Mathematics, Universit\"at Potsdam, Potsdam, Germany
\mailsc
}

\maketitle

\begin{abstract}
Transductive learning considers situations when a learner observes $m$ labelled training points and $u$ unlabelled test points with the final goal of giving correct answers for the test points.
This paper introduces a new complexity measure for transductive learning called \emph{Permutational Rademacher Complexity} (PRC) and studies its properties.
A novel symmetrization inequality is proved, which shows that PRC provides a tighter control over expected suprema of empirical processes compared to what happens in the standard i.i.d. setting.
A number of comparison results are also provided, which show the relation between PRC and other popular complexity measures used in statistical learning theory, including Rademacher complexity and Transductive Rademacher Complexity (TRC).
We argue that PRC is a more suitable complexity measure for transductive learning.
Finally, these results are combined with a standard concentration argument to provide novel data-dependent risk bounds for transductive learning.
\keywords{Transductive Learning, Rademacher Complexity, Statistical Learning Theory, Empirical Processes, Concentration Inequalities}
\end{abstract}

\section{Introduction} 
Rademacher complexities (\cite{KP99}, \cite{BM02}) play an important role in the widely used concen\-tration-based approach to statistical learning theory \cite{BBL05}, which is closely related to the analysis of empirical processes \cite{Vaar00}.
They measure a complexity of function classes and provide data-dependent risk bounds in the standard i.i.d. framework of inductive learning, thanks to symmetrization and concentration inequalities.
Recently, a number of attempts were made to apply this machinery also to the \emph{transductive learning} setting \cite{Vap98}.
In particular, the authors of \cite{EP09} introduced a notion of \emph{transductive Rademacher complexity} and provided an extensive study of its properties, as well as general transductive risk bounds based on this new complexity measure.

In the transductive learning, a learner observes $m$ labelled training points and $u$ unlabelled test points. 
The goal is to give correct answers on the test points. 
Transductive learning naturally appears in many modern large-scale applications, including text mining, recommender systems, and computer vision, where often the objects to be classified are available beforehand.
There are two different settings of transductive learning, defined by V.\,Vapnik in his book \cite[Chap. 8]{Vap98}.
The first one assumes that all the objects from the training and test sets are generated i.i.d. from an unknown distribution $P$.
The second one is \emph{distribution free}, and it assumes that the training and test sets are realized by a uniform and random partition of a fixed and finite general population of cardinality $N:= m+u$ into two disjoint subsets of cardinalities $m$ and~$u$;
moreover, no assumptions are made regarding the underlying source of this general population.
The second setting has gained much attention\footnote{
For the extensive overview of transductive risk bounds we refer the reader to \cite{P08}.
} (\cite{Vap98}, \cite{DEM04}, \cite{CM06}, \cite{EP09}, \cite{CMP+09}, and \cite{TBK14}), probably due to the fact that any upper risk bound for this setting directly implies a risk bound also for the first setting \cite[Theorem 8.1]{Vap98}.
In essence, the second setting studies uniform deviations of risks computed on two disjoint finite samples.
Following Vapnik's discussion in \cite[p. 458]{CSZ06}, we would also like to emphasize that the second setting of transductive learning naturally appears as a middle step in proofs of the standard inductive risk bounds, as a result of symmetrization or the so-called \emph{double-sample} trick.
This way better transductive risk bounds also translate into better inductive ones.

An important difference between the two settings discussed above lies in the fact that the $m$ elements of the training set in the second setting are interdependent, because they are sampled uniformly \emph{without replacement} from the general population.
As a result, the standard techniques developed for inductive learning, including concentration and Rademacher complexities mentioned in the beginning, can not be applied in this setting, since they are heavily based on the i.i.d. assumption.
Therefore, it is important to study empirical processes in the setting of sampling without replacement.

{\bf Previous work.}
A large step in this direction was made in \cite{EP09}, 
where the authors presented a version of McDiarmid's bounded difference inequality~\cite{BLM13} for sampling without replacement together with the Transductive Rademacher Complexity (TRC).
As a main application the authors derived an upper bound on the binary test error of a transductive learning algorithm in terms of TRC.
However, the analysis of \cite{EP09} has a number of shortcomings.
Most importantly, TRC depends on the unknown labels of the test set.
In order to obtain computable risk bounds, the authors resorted to the contraction inequality \cite{LT91}, which is known to be a loose step \cite{M14}, since it destroys any dependence on the labels.

Another line of work was presented in \cite{TBK14}, where variants of Talagrand's concentration inequality were derived for the setting of sampling without replacement. 
These inequalities were then applied to achieve transductive risk bounds with fast rates of convergence $o(m^{-1/2})$, following a \emph{localized} approach~\cite{BBM05}.
In contrast, in this work we consider only the worst-case analysis based on the \emph{global} complexity measures.
An analysis under additional assumptions on the problem at hand, including Mammen-Tsybakov type low noise conditions \cite{BBL05}, is an interesting open question and left for future work.

{\bf Summary of our results.}
This paper continues the analysis of empirical processes indexed by arbitrary classes of uniformly bounded functions in the setting of sampling without replacement, initiated by \cite{EP09}.
We introduce a new complexity measure called \emph{permutational Rademacher complexity} (PRC) and argue that it captures the nature of this setting very well.
Due to space limitations we present the analysis of PRC only for the special case when the training and test sets have the same size $m=u$,
which is nonetheless sufficiently illustrative\footnote{
All the results presented in this paper are also available for the general $m\neq u$ case, but we defer them to a future extended version of this paper.
}.

We prove a novel symmetrization inequality (Theorem \ref{thm:tsym}), which shows that the expected PRC and the expected suprema of empirical processes when sampling without replacement are equivalent up to multiplicative constants.
Quite remarkably, the new upper and lower bounds (the latter is often called \emph{desymmetrization inequality}) both hold without any additive terms when $m=u$, in contrast to the standard i.i.d. setting, where an additive term of order $O(m^{-1/2})$ is unavoidable in the lower bound. 
For TRC even the upper symmetrization inequality \cite[Lemma 4]{EP09} includes an additive term of the order $O(m^{-1/2})$ and no desymmetrization inequality is known.
This suggests that PRC may be a more suitable complexity measure for transductive learning.
We would also like to note that the proof of our new symmetrization inequality is surprisingly simple, compared to the one presented in \cite{EP09}.

Next we compare PRC with other popular complexity measures used in statistical learning theory.
In particular, we provide achievable upper and lower bounds, relating PRC to the conditional Rademacher complexity (Theorem~\ref{thm:dagstuhl-relation}).
These bounds show that the PRC is upper and lower bounded by the conditional Radema\-cher complexity up to additive terms of orders $o(m^{-1/2})$ and $O(m^{-1/2})$ respectively, which are achievable (Lemma \ref{lemma:achievable}).
In addition to this, Theorem~\ref{thm:dagstuhl-relation} also significantly improves bounds on the complexity measure called \emph{maximum discrepancy} presented in \cite[Lemma 3]{BM02}.
We also provide a comparison between expected PRC and TRC (Corollary \ref{corollary:comparison-prc-trc}), which shows that their values are close up to small multiplicative constants and additive terms of order $O(m^{-1/2})$.

Finally, we apply these results to obtain a new computable data-dependent risk bound for transductive learning based on the PRC (Theorem \ref{thm:risk-bound}), which holds for any bounded loss functions. 
We conclude by discussing the advantages of the new risk bound over the previously best known one of \cite{EP09}.

\section{Notations}
We will use calligraphic symbols to denote sets, with subscripts indicating their cardinalities: $\card(\Z_m)=m$.   
For any function $f$ we will denote its average value computed on a finite set $S$ by $\fa(S)$.
In what follows we will consider an arbitrary space $\mathcal{Z}$ (for instance, a space of input-output pairs) and class $F$ of functions (for instance, loss functions) mapping $\mathcal{Z}$ to $\R$.
Most of the proofs are deferred to the last section for improved readability.

Arguably, one of the most popular complexity measures used in statistical learning theory is the Rademacher complexity (\cite{LT91}, \cite{KP99}, \cite{BM02}):
\begin{definition}[Conditional Rademacher complexity]
Fix any subset $\Z_m = \{Z_1,\dots,Z_m\} \subseteq \mathcal{Z}$.
The following random quantity is commonly known as a \emph{conditional Rademacher complexity}:
\[
\Radc_m(F, \Z_m)
=
\E_{\vec \epsilon}
\left[
\frac{2}{m}
\sup_{f\in F}
\sum_{i=1}^m \epsilon_i f(Z_i)
\right],\enspace
\]
where $\vec \epsilon = \{\epsilon_i\}_{i=1}^m$ are i.i.d. Rademacher signs, taking values $\pm 1$ with probabilities $1/2$.
When the set $\Z_m$ is clear from the context we will simply write $\Radc_m(F)$.
\end{definition}
As discussed in the introduction, Rademacher complexities play an important role in the analysis of empirical processes and statistical learning theory.
However, this measure of complexity was devised mainly for the i.i.d. setting, which is different from our setting of sampling without replacement.
The following complexity measure was introduced in \cite{EP09} to overcome this issue:
\begin{definition}[Transductive Rademacher complexity]
\label{def:trc}
Fix any set $\Z_{N} = \{Z_1,\dots,Z_{N}\} \subseteq \mathcal{Z}$, positive integers $m,u$ such that $N=m+u$, and $p \in \left[0, \frac{1}{2}\right]$.
The following quantity is called \emph{Transductive Rademacher complexity} (TRC):
\[
\Radctd_{m+u}(F, \Z_{N}, p) = 
\left(\frac{1}{m} + \frac{1}{u}\right)
\E_{\vec \sigma}
\left[
\sup_{f\in F}
\sum_{i=1}^{N} \sigma_i f(Z_i)
\right],\enspace
\]
where $\vec \sigma = \{\sigma_1\}_{i=1}^{m+u}$ are i.i.d. random variables taking values $\pm1$ with probabilities $p$ and $0$ with probability $1 - 2p$.
\end{definition}
We summarize the importance of these two complexity measures in the analysis of empirical processes when sampling without replacement in the following result:
\begin{theorem}
\label{thm:overview}
Fix an \hbox{$N$-element} subset $\Z_N\subseteq \Z$ and let $m<N$ elements of $\Z_m$ be sampled uniformly without replacement from $\Z_N$.
Also let $m$ elements of $\X_m$ be sampled uniformly with replacement from $\Z_N$.
Denote $\Z_u := \Z_N \setminus \Z_m$ with $u:=\card(\Z_u)= N-m$.
The following upper bound in terms of the i.i.d. Rademacher complexity was provided in \cite{TBK14}:
\begin{equation}
\label{eq:tsym-tbk14}
\E_{\Z_m}\sup_{f\in F} \left( \fa(\Z_u) - \fa(\Z_m) \right)
\leq
\frac{N}{u}\cdot
\E_{\X_m}\left[ \Radc_m(F,\X_m)\right].\enspace
\end{equation}
The following bound in terms of TRC was provided in \cite{EP09}.
Assume that functions in $F$ are uniformly bounded by $B$. 
Then for $p_0:=\frac{mu}{N^2}$ and $c_0 < 5.05$:
\begin{equation}
\label{eq:tsym-ep09}
\E_{\Z_m}\sup_{f\in F} \left( \fa(\Z_u) - \fa(\Z_m) \right)
\leq
\Radctd_{m+u}(F, \Z_{N}, p_0)
+
c_0 B \frac{N\sqrt{\min(m,u)}}{mu}.\enspace
\end{equation}
\end{theorem}
While \eqref{eq:tsym-tbk14} did not explicitly appear in \cite{TBK14}, it can be immediately derived using \cite[Corollary 8]{TBK14} and i.i.d. symmetrization of \cite[Theorem 2.1]{K11sf}.

Finally, we introduce our new complexity measure:
\begin{definition}[Permutational Rademacher complexity]
\label{def:PRC}
Let $\Z_m \subseteq \mathcal{Z}$ be any fixed set of cardinality $m$.
For any $n\in \{1,\dots,m-1\}$ the following quantity will be called a \emph{permutational Rademacher complexity (PRC)}:
\[
\PRC_{m,n}(F,\Z_m)
=
\E_{\Z_n}
\sup_{f\in F}
\left(
\fa(\Z_k)
-
\fa(\Z_n)
\right),\enspace
\]
where $\Z_n$ is a random subset of $\Z_m$ containing $n$ elements sampled uniformly without replacement and $\Z_k := \Z_m\setminus \Z_n$.
When the set $\Z_m$ is clear from the context we will simply write $\PRC_{m,n}(F)$.
\end{definition}

The name PRC is explained by the fact that if $m$ is even then the definitions of $\PRC_{m,m/2}(F)$ and $\Radc_{m}(F)$ are very similar. 
Indeed, the  only difference is that the expectation in the PRC is over the randomly permuted sequence containing \emph{equal number} of $``-1"$ and $``+1"$, whereas in Rademacher complexity the average is w.r.t. all the possible sequences of signs.
The term ``permutation complexity" has already appeared in \cite{M10}, where it was used to denote a novel complexity measure for a model selection.
However, this measure was specific to the i.i.d. setting and \emph{binary} loss. 
Moreover, the bounds presented in \cite{M10} were of the same order as the risk bounds based on the Rademacher complexity with worse constants in the slack term.

\section{Symmetrization and Comparison Results}
\label{section:results}
We start with showing a version of the i.i.d. symmetrization inequality (references can be found in \cite{LT91}, \cite{K11sf}) for the setting of sampling without replacement.
It shows that the expected supremum of empirical processes in this setting is up to multiplicative constants equivalent to the expected PRC.
\begin{theorem}
\label{thm:tsym}
Fix an \hbox{$N$-element} subset $\Z_N\subseteq \Z$ and let $m<N$ elements of $\Z_m$ be sampled uniformly without replacement from $\Z_N$.
Denote $\Z_u := \Z_N \setminus \Z_m$ with $u:=\card(\Z_u)= N-m$.
If $m=u$ and $m$ is even then
 for any $n\in\{1,\dots,m-1\}$:
\[
\frac{1}{2}
\E_{\Z_m}\left[\PRC_{m,m/2}(F,\Z_m)\right]
\leq
\E_{\Z_m}\sup_{f\in F} \left( \fa(\Z_u) - \fa(\Z_m) \right)
\leq
\E_{\Z_m}\left[\PRC_{m,n}(F,\Z_m)\right].\enspace
\]
The inequalities also hold if we include absolute values inside the suprema.
\end{theorem}
\begin{proof}
The proof can be found in Sect. \ref{sect:proof-symmetrization}.
\end{proof}
This inequality should be compared to the previously known complexity bounds of Theorem~\ref{thm:overview}.
First of all, in contrast to \eqref{eq:tsym-tbk14} and \eqref{eq:tsym-ep09} the new bound provides a two sided control, which shows that PRC is a ``correct" complexity measure for our setting.
It is also remarkable that the lower bound (commonly known as the \emph{desymmetrization inequality}) does not include any additive terms,
since in the standard i.i.d. setting the lower bound holds only up to an additive term of order $O(m^{-1/2})$ \cite[Sect. 2.1]{K11sf}.
Also note that this result does not assume the boundedness of functions in $F$, which is a necessary assumptions both in  \eqref{eq:tsym-ep09} and in the i.i.d. desymmetrization inequality.

Next we compare PRC with the conditional Rademacher complexity:
\begin{theorem}
\label{thm:dagstuhl-relation}
Let $\Z_m \subseteq \mathcal{Z}$ be any fixed set of even cardinality $m$.
Then:
\begin{equation}
\label{eq:dagstuhl-1}
\PRC_{m,m/2}(F,\Z_m) 
\leq 
\left(1 + \frac{2}{\sqrt{2\pi m} - 2}\right)\,\Radc_m(F,\Z_m).\enspace
\end{equation}
Moreover, if the functions in $F$ are absolutely bounded by $B$ then
\begin{equation}
\label{eq:dagstuhl-2}
\left|
\PRC_{m,m/2}(F,\Z_m)
-
\Radc_m(F,\Z_m)
\right|
\leq
\frac{2B}{\sqrt{m}}.\enspace
\end{equation}
The results also hold if we include absolute values inside suprema in $\PRC_{m,n}, \Radc_m$.
\end{theorem}
\begin{proof}
Conceptually the proof is  based on the coupling between a sequence $\{\epsilon_i\}_{i=1}^{m}$ of i.i.d. Rademacher signs and a uniform random permutation $\{\eta_i\}_{i=1}^m$ of a set containing $m/2$ plus and $m/2$ minus signs.
This idea was inspired by the techniques used in \cite{GN10}.
The detailed proof can be found in Sect. \ref{sect:proof-dagstuhl}.
\end{proof}

Note that a typical order of $\Radc_m(F)$ is $O(m^{-1/2})$, thus the multiplicative upper bound \eqref{eq:dagstuhl-1} can be much tighter than the upper bound of \eqref{eq:dagstuhl-2}.
We would also like to note that Theorem \ref{thm:dagstuhl-relation} significantly improves bounds of Lemma 3 in~\cite{BM02}, which relate the so-called \emph{maximal discrepancy} measure of the class $F$ to its Rademacher complexity (for the further discussion we refer to Appendix).


Our next result shows that bounds of Theorem \ref{thm:dagstuhl-relation} are essentially tight.
\begin{lemma}
\label{lemma:achievable}
Let $\Z_m\subseteq \Z$ with even $m$. There are two finite classes $F_m'$ and $F_m''$ of functions mapping $\Z$ to $\R$ and absolutely bounded by $1$, such that:
\begin{equation}
\label{eq:counterex-1}
\PRC_{m,m/2}(F'_m,\Z_m) =0,\quad
(2 m)^{-1/2}
\leq
\Radc_m(F'_m,\Z_m)
\leq
2 m^{-1/2};\enspace
\end{equation}
\begin{equation}
\label{eq:counterex-2}
\PRC_{m,m/2}(F''_m,\Z_m) = 1,\quad
1 - \sqrt{\frac{{2}}{{\pi m}}}
\leq
\Radc_m(F''_m,\Z_m)
\leq
1 - \frac45 \sqrt{\frac{2}{\pi m}}.\enspace
\end{equation}
\end{lemma}
\begin{proof}
The proof can be found in Sect. \ref{sect:proof-lemma-counterex}.
\end{proof}
Inequalities \eqref{eq:counterex-1} simultaneously show that (a) the order $O(m^{-1/2})$ of the additive bound \eqref{eq:dagstuhl-2} can not be improved, and (b) the multiplicative upper bound \eqref{eq:dagstuhl-1} can not be reversed.
Moreover, it can be shown using \eqref{eq:counterex-2} that the factor appearing in \eqref{eq:dagstuhl-1} can not be improved to $1 + o(m^{-1/2})$.

Finally, we compare PRC to the transductive Rademacher complexity:
\begin{lemma}
\label{lem:comparison}
Fix any set $\Z_{N} = \{Z_1,\dots,Z_{N}\} \subseteq \mathcal{Z}$.
If $m=u$ and $N=m+u$:
\[
\Radc_{N}(F, \Z_{N}) 
\leq 
\Radctd_{m+u}\left(F, \Z_{N}, 1/4\right) 
\leq 2\Radc_{N}(F, \Z_{N}).\enspace
\]
\end{lemma}
\begin{proof}
The upper bound was presented in \cite[Lemma 1]{EP09}.
For the lower bound, notice that if $p=1/4$ the i.i.d. signs $\sigma_i$ presented in Definition \ref{def:trc} have the same distribution as $\epsilon_i \eta_i$, where $\epsilon_i$ are i.i.d. Rademacher signs and $\eta_i$ are i.i.d. Bernoulli random variables with parameters $1/2$.
Thus, Jensen's inequality gives:
\begin{align*}
\Radctd_{m+u}\left(F, \Z_{N}, 1/4\right) 
=
\frac{4}{N}
\E_{(\vec \epsilon, \vec \eta)}
\left[
\sup_{f\in F}
\sum_{i=1}^{m + u} \epsilon_i \eta_i f(Z_i)
\right]
\geq
\frac{4}{N}
\E_{\vec \epsilon}
\left[
\sup_{f\in F}
\sum_{i=1}^{m + u} \epsilon_i \frac12 f(Z_i)
\right].\enspace
\end{align*}
\end{proof}

Together with Theorems \ref{thm:tsym} and \ref{thm:dagstuhl-relation} this result shows that when $m=u$ the PRC can not be much larger than transductive Rademacher complexity:
\begin{corollary}
\label{corollary:comparison-prc-trc}
Using notations of Theorem \ref{thm:tsym}, we have:
\[
\E_{\Z_m}\left[\PRC_{m,m/2}(F,\Z_m)\right]
\leq
\left(2 + \frac{4}{\sqrt{2\pi N} - 2}\right)\,\Radctd_{m+u}(F, \Z_{N}, 1/4).\enspace
\]
If functions in $F$ are uniformly bounded by $B$ then we also have a lower bound:
\[
\E_{\Z_m}\left[\PRC_{m,m/2}(F,\Z_m)\right]
\geq
\frac12\Radctd_{m+u}(F, \Z_{N}, 1/4) + \frac{2B}{\sqrt{N}}.\enspace
\]
\end{corollary}
\begin{proof}
Simply notice that $\E_{\Z_m}\left[\sup_{f\in F} \left( \fa(\Z_u) - \fa(\Z_m) \right)\right] = \PRC_{N,m}(F,\Z_N)$.
\end{proof}

\section{Transductive Risk Bounds}
Next we will use the results of Sect. \ref{section:results} to obtain a new transductive risk bound.
First we will shortly describe the setting.

We will consider the second, distribution-free setting of transductive learning described in the introduction.
Fix any finite \emph{general population} of input-output pairs $\Z_N = \{(x_i,y_i)\}_{i=1}^N \subseteq \X\times\Y$, where $\X$ and $\Y$ are arbitrary input and output spaces.
We make no assumptions regarding underlying source of $\Z_N$.
The learner receives the labeled \emph{training set} $\Z_m$ consisting of $m<N$ elements sampled uniformly without replacement from $\Z_N$.
The remaining \emph{test set} $\Z_u := \Z_N \setminus \Z_m$ is presented to the learner \emph{without labels} (we will use $\X_u$ to denote the inputs of $\Z_u$).
The goal of the learner is to find a predictor in the fixed \emph{hypothesis class} $\Hyp$ based on the training sample $\Z_m$ and unlabelled test points $\X_u$, which has a small test risk measured using bounded \emph{loss function} $\ell\colon \Y\times\Y\to[0,1]$.
For $h\in \Hyp$ and $(x,y)\in\Z_N$ denote $\ell_h(x,y) = \ell\bigl(h(x),y\bigr)$ and also denote the loss class $L_{\Hyp} = \{\ell_h \colon h\in\Hyp\}$.
Then the test and training risks of $h\in\Hyp$ are defined as $\err_u(h) := \overline{\ell_h}(\Z_u)$ and $\err_m(h) : =\overline{\ell_h}(\Z_m)$ respectively.

Following risk bound in terms of TRC was presented in \cite[Corollary 2]{EP09}:
\begin{theorem}[\cite{EP09}]
\label{thm:risk-bound-ep09}
If $m=u$ then with probability at least $1-\delta$ over the random training set $\Z_m$ any $h\in\Hyp$ satisfies:
\begin{equation}
\label{eq:ep-risk-bound}
\err_u(h) 
\leq
\err_m(h)
+
\Radctd_{m+u}\left(L_{\Hyp}, \Z_{N}, 1/4\right) 
+
11\sqrt{\frac{2}{N}} + \sqrt{\frac{2N\log(1/\delta)}{(N-1/2)^2}}.\enspace
\end{equation}
\end{theorem}
Using results of Sect. \ref{section:results} we obtain the following risk bound:
\begin{theorem}
\label{thm:risk-bound}
If $m=u$ and $n\in\{1,\dots,m-1\}$ then with probability at least $1-\delta$ over the random training set $\Z_m$ any $h\in\Hyp$ satisfies:
\begin{equation}
\label{eq:new-risk-bound}
\err_u(h) 
\leq
\err_m(h) 
+
\E_{\S_m}\left[\PRC_{m,n}(L_{\Hyp},\Z_m)\right] + \sqrt{\frac{2N\log(1/\delta)}{(N-1/2)^2}}.\enspace
\end{equation}
Moreover, with probability at least $1-\delta$ any $h\in\Hyp$ satisfies:
\begin{equation}
\label{eq:new-risk-bound-dd}
\err_u(h) 
\leq
\err_m(h) 
+
\PRC_{m,n}(L_{\Hyp},\Z_m) + 2\sqrt{\frac{2N\log(2/\delta)}{(N-1/2)^2}}.\enspace
\end{equation}
\end{theorem}
\begin{proof}
The proof can be found in Sect. \ref{sect:proof-risk-bound}.
\end{proof}
We conclude by comparing risk bounds of Theorems \ref{thm:risk-bound} and \ref{thm:risk-bound-ep09}:

1. First of all, the upper bound of \eqref{eq:new-risk-bound-dd} is computable.
This bound is based on the concentration argument, which shows that the expected PRC (appearing in~\eqref{eq:new-risk-bound}) can be nicely estimated using the training set.
Meanwhile, the upper bound of~\eqref{eq:ep-risk-bound} depends on the \emph{unknown} labels of the test set through TRC.
In order to make it computable the authors of \cite{EP09} resorted to the contraction inequality, which allows to drop any dependence on the labels for Lipschitz losses, which is known to be a loose step \cite{M14}.

2. Moreover, we would like to note that for binary loss function TRC (as well as the Rademacher complexity) does not depend on the labels at all.
Indeed, this can be shown by writing $\ell_{01}(y,y') = (1 - y y')/2$ for $y,y'\in\{-1,+1\}$ and noting that $\sigma_i$ and $\sigma_i y$ are identically distributed for $\sigma_i$ used in Definition~\ref{def:trc}.
This is not true for PRC, which is \emph{sensitive} to the labels even in this setting.
As a future work we hope to use this fact for analysis in the low noise setting \cite{BBL05}.

3. The slack term appearing in \eqref{eq:new-risk-bound} is significantly smaller than the one of~\eqref{eq:ep-risk-bound}.
For instance, if $\delta = 0.01$ then the latter is $13$ times larger.
This is caused by the additive term in symmetrization inequality \eqref{eq:tsym-ep09}.
At the same time, Corollary \ref{corollary:comparison-prc-trc} shows that the complexity term appearing in \eqref{eq:new-risk-bound} is at most two times larger than TRC, appearing in \eqref{eq:ep-risk-bound}.

4. Comparison result of Theorem \ref{thm:dagstuhl-relation} shows that the upper bound of \eqref{eq:new-risk-bound-dd} is also tighter than the one which can be obtained using \eqref{eq:tsym-tbk14} and conditional Rademacher complexity.

5. Similar upper bounds (up to extra factor of 2) also hold for the \emph{excess risk} $\err_u(h_m)- \inf_{h\in\Hyp} \err_u(h)$, where $h_m$ minimizes the training risk $\err_m$ over $\Hyp$. This can be proved using a similar argument to Theorem \ref{thm:risk-bound}.

6. Finally, one more application of the concentration argument can simplify the computation of PRC, by estimating the expected value appearing in Definition \ref{def:PRC} with only one random partition of $\Z_m$.


\section{Full Proofs}
\subsection{Proof of Theorem \ref{thm:tsym}}
\label{sect:proof-symmetrization}
\begin{lemma}
\label{lemma:Mean}
For $0<m\leq N$ let $\S_m:=\{s_1,\dots,s_m\}$ be sampled uniformly without replacement from a finite set of real numbers $\C=\{c_1,\dots,c_N\}\subset \mathbb{R}$.
Then:
\[
\E_{\S_m}\left[\frac{1}{m}\sum_{i=1}^m s_i\right]
=
\frac{1}{{N \choose m}}\sum_{\S_m \subseteq \C}\frac{1}{m}\sum_{z\in \S_m} z
=
\frac{1}{m {N \choose m}}\sum_{i=1}^N \binom{N-1}{m-1} c_i
=
\frac{1}{N} \sum_{i=1}^N c_i.
\]
\end{lemma}
\begin{proof}[of Theorem \ref{thm:tsym}]
Fix any \emph{positive} integers $n$ and $k$ such that $n + k = m$, which implies $n<m$ and $k < m = u$.
Note that Lemma \ref{lemma:Mean} implies:
\[
\fa(\Z_u)= 
\E_{\S_k}\left[\fa(\S_k)\right],\quad
\fa(\Z_m)= 
\E_{\S_n}\left[\fa(\S_n)\right],
\]
where $\S_k$ and $\S_n$ are sampled uniformly without replacement from $\Z_u$ and $\Z_m$ respectively.
Using Jensen's inequality we get:
\begin{align}
\notag
\E_{\Z_m}\sup_{f\in F} \left( \fa(\Z_u) - \fa(\Z_m) \right)
&=
\E_{\Z_m}
\sup_{f\in F}
\left(
\E_{\S_k}\left[\fa(\S_k)\right] - \E_{\S_n}\left[\fa(\S_n)\right]
\right)\\
\label{eq:blabla}
&\leq
\E_{(\Z_m,\S_k,\S_n)}
\sup_{f\in F}
\bigl(\fa(\S_k) - \fa(\S_n)\bigr).
\end{align}
The marginal distribution of $(\S_k,\S_n)$, appearing in \eqref{eq:blabla}, can be equivalently described by first sampling $\Z_m$ from $\Z_N$, then $\S_n$ from $\Z_m$ (both times uniformly without replacement), and setting $\S_k := \Z_m\setminus \S_n$ (recall that $n+k=m$).
Thus
\[
\E_{(\Z_m,\S_k,\S_n)}
\sup_{f\in F}
\bigl(\fa(\S_k) - \fa(\S_n)\bigr)
=
\E_{\Z_m}\left[
\E_{\S_n}\left[
\sup_{f\in F}
\bigl(\fa(\Z_m\setminus \S_n) - \fa(\S_n)\bigr) \bigg|
\Z_m \right]
\right],
\]
which completes the proof of the upper bound.

We have shown that for $n\in\{1,\dots,m-1\}$ and $k:=m-n$:
\begin{equation}
\label{eq:tsym-proof-step-1}
\E_{\Z_m}\left[\PRC_{m,n}(F,\Z_m)\right]
=
\E_{(\Z_k,\Z_n)}
\sup_{f\in F}
\bigl(\fa(\Z_k) - \fa(\Z_n)\bigr),
\end{equation}
where $\Z_n$ and $\Z_k$ are sampled uniformly without replacement from $\Z_N$ and $\Z_N\setminus \Z_n$ respectively.
Let $\Z_{m-n}$ be sampled uniformly without replacement from $\Z_N\setminus (\Z_n\cup \Z_k)$ and let $\Z_{u-k}$ be the remaining $u-k$ elements of $\Z_N$.
Using Lemma \ref{lemma:Mean} once again we get:
\[
\E\left[\fa(\Z_{m-n}) \big| (\Z_n,\Z_k)\right]
=
\E\left[\fa(\Z_{u-k})\big| (\Z_n,\Z_k)\right].
\]
We can rewrite the r.h.s.\:of \eqref{eq:tsym-proof-step-1} as:
\begin{align*}
&\E_{(\Z_n,\Z_k)}
\sup_{f\in F}
\left(\fa(\Z_k) - \fa(\Z_n) + 
\E\left[
\fa(\Z_{u-k}) - \fa(\Z_{m-n})\big| (\Z_n,\Z_k)\right]
\right)\\
&\leq
\E
\sup_{f\in F}
\left(\fa(\Z_k) - \fa(\Z_n) + 
\fa(\Z_{u-k})
-
\fa(\Z_{m-n})
\right),
\end{align*}
where we have used Jensen's inequality.
If we take $n^* = k^* = m/2$ we get
\[
\E_{\Z_m}\left[\PRC_{m,m/2}(F,\Z_m)\right]
\leq
\E
\sup_{f\in F}
\left(
2\fa(\Z_{k^*} \cup \Z_{u-k^*}) - 
2\fa(\Z_{n^*} \cup \Z_{m-n^*})
\right).
\]
It is left to notice that the random subsets $\Z_{k^*} \cup \Z_{u-k^*}$ and $\Z_{n^*} \cup \Z_{m-n^*}$ have the same distributions as $\Z_u$ and $\Z_m$.
\end{proof}

\subsection{Proof of Theorem \ref{thm:dagstuhl-relation}}
\label{sect:proof-dagstuhl}
Let $m=2\cdot n$,
$\vepsilon = \{\epsilon_i\}_{i=1}^{m}$ be i.i.d.\:Rademacher signs, 
and $\veta = \{\eta_i\}_{i=1}^{m}$ be a uniform random permutation of a set containing $n$ plus and $n$ minus signs.
The proof of Theorem \ref{thm:dagstuhl-relation} is based on the coupling of random variables $\vepsilon$ and $\veta$, which is described in Lemma \ref{lemma:coupling}.
We will need a number of definitions.
Consider binary cube $B_m:= \{-1,+1\}^m$.
Denote $S_m:=\left\{v\in B_m\colon \sum_{i=1}^m v_i = 0\right\}$,
which is a set of all the vectors in $B_m$ having equal number of plus and minus signs.
For any $v\in B_m$ denote $\|v\|_1 = \sum_{i=1}^m |v_i|$ and consider the following set:
\[
T(v) = \arg \min_{v' \in S_m} \|v - v'\|_1,
\]
which consists of the points in $S_m$ closest to $v$ in Hamming metric.
For any $v\in B_m$ let $t(v)$ be a random element of $T(v)$, distributed uniformly.
We will use $t_i(v)$ to denote $i$-th coordinate of the vector $t(v)$.
\begin{remark}
\label{remark:coupling}
If $v\in S_m$ then $T(v) = \{v\}$.
Otherwise, $T(v)$ will clearly contain more than one element of $S_m$.
Namely, it can be shown, that if for some positive integer $q$ it holds that $\sum_{i=1}^m v_i = q$, then $q$ is necessarily even and $T(v)$ consists of all the vectors in $S_m$ which can be obtained by replacing $q/2$ of $+1$ signs in $v$ with $-1$ signs, and thus in this case $\mathrm{card}\bigl(T(v)\bigr) = {(m+q)/2 \choose q/2}$.
\end{remark}
\begin{lemma}[Coupling]
\label{lemma:coupling}
Assume that $m = 2\cdot n$. Then the random sequence~$t(\vec \epsilon)$ has the same distribution as $\veta$.
\end{lemma}
\begin{proof}
Note that the support of $t(\vec \epsilon)$ is equal to $S_m$.
From symmetry it is easy to conclude that the distribution of $t(\vec \epsilon)$ is exchangable. 
This means that it is invariant under permutations and as a consequence uniform on $S_m$.
\end{proof}

Next result is in the core of the multiplicative upper bound \eqref{eq:dagstuhl-1}.
\begin{lemma}
\label{lemma:coupling-expect}
Assume that $m=2\cdot n$.
For any $q\in\{1,\dots,m\}$ the following holds:
\[
\E[\epsilon_q | t(\vepsilon)]
=
\left(1 - 2^{-m}{m \choose n}\right)  t_q(\vepsilon)
\geq
\left(1 - 2(2\pi m)^{-1/2}\right) t_q(\vepsilon).
\]
\end{lemma}
\begin{proof}
We will first upper bound $\Prob\{\epsilon_q \neq t_q(\vepsilon) | t(\vepsilon) = \vec e\}$, where $\vec e = \{e_i\}_{i=1}^{m}$ is (w.l.o.g.) a sequence of $n$ plus signs followed by a sequence of $n$ minus signs.
\begin{align}
\notag
\Prob\{\epsilon_q \neq t_q(\vepsilon) | t(\vepsilon) = \vec e\}
&=
\frac{\Prob\{\epsilon_q \neq t_q(\vepsilon) \cap t(\vepsilon) = \vec e\}}{ \Prob\{ t(\vepsilon) = \vec e\}}\\
\label{eq:coupling-1}
&=
{m \choose n}2^{-m} \sum_{\vec s} \Prob\{\epsilon_q \neq t_q(\vepsilon) \cap t(\vepsilon) = \vec e | \vepsilon = \vec s\},
\end{align}
where we have used Lemma \ref{lemma:coupling} and the sum is over all different sequences of $m$ signs $\vec s = \{s_i\}_{i=1}^{m}$.
For any $\vec s$ denote $S(\vec s) = \sum_{j=1}^n s_j$ and consider terms in \eqref{eq:coupling-1} corresponding to $\vec s$ with $S(\vec s) = 0$, $S(\vec s) > 0$, and $S(\vec s) < 0$:

{\bf Case 1:} $S(\vec s) = 0$. These terms will be zero, since $t(\vec s) = \vec s$.

{\bf Case 2:} $S(\vec s) > 0$.
This means that $\vec s$ ``has more plus signs than it should'' and according to Remark~\ref{remark:coupling} the mapping $t(\cdot)$ will replace several of ``+1" with ``-1".
In particular, if $s_q = -1$ then $t_q(\vec s) = s_q$ and thus the corresponding terms will be zero. 
If $s_q = 1$ and in the same time $e_q = 1$ the event $\{\epsilon_q \neq t_q(\vepsilon) \cap t(\vepsilon) = \vec e\}$ also can not hold.
Moreover, note that identity $\vec e = t(\vec s)$ can hold only if $\vec e \in T(\vec s)$, which necessarily leads to 
\begin{equation}
\label{eq:condition-couple}
\bigl\{j\in\{1,\dots,m\}\colon s_j = -1\bigr\}\subseteq \bigl\{j\in\{1,\dots,m\} \colon e_j = -1\bigr\}.
\end{equation}
From this we conclude that if $q\in\{1,\dots,n\}$ then all the terms corresponding to $\vec s$ with $S(\vec s)>0$ are zero.
We will use $U_q(\vec e)$ to denote the subset of $B_m$ consisting of sequences $\vec s$, such that (a) $S(\vec s) >0$, (b) $s_q = 1$, and (c) condition \eqref{eq:condition-couple} holds.
It can be seen that if $\vec s \in U_q(\vec e)$ then:
\[
\Prob\{\epsilon_q \neq t_q(\vepsilon) \cap t(\vepsilon) = \vec e | \vepsilon = \vec s\} = {n + S(\vec s)/2 \choose S(\vec s)/2}^{-1}.
\]
This holds since, according to Remark~\ref{remark:coupling}, $t(\vepsilon)$ can take exactly ${n + S(\vec s)/2 \choose S(\vec s)/2}$ different values, while only one of them is equal to $\vec e$.

Let us compute the cardinality of $U_q(\vec e)$ for $q\in\{n+1,\dots,m\}$.
It is easy to check that condition $S(\vec s) = 2j$ for some positive integer $j$ implies that $\vec s$ has exactly $n - j$ minus signs.
Considering the fact that $s_q = 1$ for $\vec s \in U_q(\vec e)$ we have:
\[
\card\bigl(U_q(\vec e)\bigr) = {n - 1 \choose n - j}.
\]
Combining everything together we have:
\[
\sum_{\vec s \colon S(\vec s)>0} 
\Prob\{\epsilon_q \neq t_q(\vepsilon) \cap t(\vepsilon) = \vec e | \vepsilon = \vec s\}
=
\mathbbm{1}\{q > n\}\,
\sum_{j=1}^n
\frac{{n - 1 \choose n - j}}{{n + j \choose j}}.
\]
Finally, it is easy to show using induction that:
\[
\sum_{j=1}^n
\frac{{n - 1 \choose n - j}}{{n + j \choose j}}
=
\frac{1}{2}.
\]
{\bf Case 3:} $S(\vec s) < 0$.
We can repeat all the steps of the previous case and get:
\[
\sum_{\vec s \colon S(\vec s)<0} \Prob\{\epsilon_q \neq t_q(\vepsilon) \cap t(\vepsilon) = \vec e | \vepsilon = \vec s\}
=
\frac12\mathbbm{1}\{q \leq n\}.
\]

Accounting for these three cases in \eqref{eq:coupling-1} we conclude that
\[
\Prob\{\epsilon_q \neq t_q(\vepsilon) | t(\vepsilon) = \vec e\}
=
\frac{1}{2}{m \choose n}2^{-m} \leq \frac{1}{\sqrt{2\pi m}},
\]
where we have used the upper bound on the binomial coefficient from \cite[Corollary 2.4]{S01}.
We can conclude the proof of lemma by writing:
\begin{align*}
\E[\epsilon_q | t(\vepsilon)]
=
t_q(\vepsilon) \left( 1 - 2\Prob\{\epsilon_q \neq t_q(\vepsilon) | t(\vepsilon)\}\right)
\geq
t_q(\vepsilon) \left( 1 - 2(2\pi m)^{-1/2}\right).
\end{align*}
\end{proof}

\begin{proof}[of Theorem \ref{thm:dagstuhl-relation}]
First we prove \eqref{eq:dagstuhl-1}.
Let $\Z_m = \{z_1,\dots,z_m\}$.
We can write:
\begin{align}
\label{eq:proof-gilles-1}
\PRC_{m,n}(F)
&=
\E\left[
\sup_{f\in F} \frac{2}{m} \sum_{i=1}^m t_i(\vec \epsilon) f(z_i)
\right]\\
\label{eq:proof-gilles-2}
&\leq
\bigl(1 - 2(2\pi m)^{-1/2}\bigr)^{-1}
\E\left[
\sup_{f\in F} \frac{2}{m} \sum_{i=1}^m \E[\epsilon_i | t(\vec \epsilon)] f(z_i)
\right]\\
\label{eq:proof-gilles-3}
&\leq
\left(1 + \frac{2}{\sqrt{2\pi m} - 2}\right)
\E\left[
\sup_{f\in F} \frac{2}{m} \sum_{i=1}^m \epsilon_i f(z_i)
\right],
\end{align}
where we have used coupling Lemma \ref{lemma:coupling} in \eqref{eq:proof-gilles-1}, Lemma \ref{lemma:coupling-expect} in \eqref{eq:proof-gilles-2}, and Jensen's inequality in \eqref{eq:proof-gilles-3}.
This completes the proof of \eqref{eq:dagstuhl-1}.

Next we prove \eqref{eq:dagstuhl-2}. 
We have:
\[
\left|
\PRC_{m,n}(F)
-
\Radc_m(F)
\right|
=
\left|
\E_{\vec \eta}\left[ \sup_{f\in F} \frac{2}{m} \sum_{i=1}^m \eta_i f(z_i)\right]
-
\E_{\vec \epsilon}\left[ \sup_{f\in F} \frac{2}{m} \sum_{i=1}^m \epsilon_i f(z_i)\right]
\right|.
\]
Using Lemma \ref{lemma:coupling} and Jensen's inequality we further get:
\begin{align}
\notag
&\left|
\PRC_{m,n}(F)
-
\Radc_m(F)
\right|\\
\notag
&=
\left|
\E_{\vec \epsilon}\left[ \E_{t}\left[\sup_{f\in F} \frac{2}{m} \sum_{i=1}^m t_i(\vec \epsilon) f(z_i)\bigg| \vec \epsilon \right]\right]
-
\E_{\vec \epsilon}\left[ \sup_{f\in F} \frac{2}{m} \sum_{i=1}^m \epsilon_i f(z_i)\right]
\right|\\
\label{eq:proof-k-1}
&\leq
\E_{\vec \epsilon}\left[
\E_{t}\left[
\left|\sup_{f\in F} \frac{2}{m} \sum_{i=1}^m t_i(\vec \epsilon) f(z_i)
-
\sup_{f\in F} \frac{2}{m} \sum_{i=1}^m \epsilon_i f(z_i)\right|\, \bigg| \vec \epsilon \right] \right],
\end{align}
where we have, perhaps misleadingly, denoted the conditional expectation with respect to~the uniform choice from $T(\vec \epsilon)$ given $\vec \epsilon$ using $\E_{t}[\,\cdot\,| \vec \epsilon]$.
Next we have:
\begin{equation}
\label{eq:proof-abs}
\left|\sup_{f\in F} \frac{2}{m} \sum_{i=1}^m t_i(\vec \epsilon) f(z_i)
-
\sup_{f\in F} \frac{2}{m} \sum_{i=1}^m \epsilon_i f(z_i)\right|
\leq
\left|\sup_{f\in F} \frac{4}{m} \sum_{i\in S(\vec \epsilon, t)} \epsilon_i f(z_i)\right|,
\end{equation}
where $S(\vec \epsilon, t) \subseteq \{1,\dots, m\}$ is a subset of indices, s.t.\:$\bigl(t(\vec \epsilon)\bigr)_i \neq \epsilon_i$ iff $i\in S(\vec \epsilon, t)$.
We can continue by writing
\begin{equation}
\label{eq:proof-k-2}
\left|\sup_{f\in F} \frac{2}{m} \sum_{i=1}^m t_i(\vec \epsilon) f(z_i)
-
\sup_{f\in F} \frac{2}{m} \sum_{i=1}^m \epsilon_i f(z_i)\right|
\leq
\frac{4}{m} \sup_{f\in F} \sum_{i \in S(\vec \epsilon, t)} |f(z_i)|.
\end{equation}
Note that since functions in $F$ are absolutely bounded by $B$:
\[
\sup_{f\in F} \sum_{i \in S(\vec \epsilon, t)} |f(z_i)| \leq B\cdot \card\left(S(\vec \epsilon, t)\right).
\]
Returning to \eqref{eq:proof-k-1} and using Remark \ref{remark:coupling} we obtain:
\[
\left|
\PRC_{m,n}(F)
-
2\Radc_m(F)
\right|
\leq
\frac{4B}{m}  \E_{\vec \epsilon}\left[ \E_t\left[\card\left(S(\vec \epsilon, t)\right) | \vec \epsilon\right] \right]
=
\E_{\vec \epsilon}\left[\frac{1}{2}\left| \sum_{i=1}^m \epsilon_i \right|\,\right].
\]
Khinchin's inequality \cite[Lemma 4.1]{LT91} together with the best known constant due to \cite{H81} gives
$
\E_{\vec \epsilon}\left[\left| \sum_{i=1}^m \epsilon_i \right|\,\right] \leq \sqrt{m},
$
which completes the proof of \eqref{eq:dagstuhl-2}.
\end{proof}

\subsection{Proof of Lemma \ref{eq:counterex-1}}
\label{sect:proof-lemma-counterex}
\begin{proof}
Let $\Z_m=\{z_1,\dots,z_m\}$.
Take $F_m'$ to be a set of two constant functions, $f_1(z)=1$ and $f_2(z)=0$ for all $z\in \Z$.
Clearly, $\PRC_{m,n}(F_m') = 0$.
In the same time:
\[
\E_{\vec \epsilon}\left[ \sup_{f\in F_m'} \frac{2}{m} \sum_{i=1}^m \epsilon_i f(z_i)\right]
=
\E_{\vec \epsilon}\left[\max\left\{0, \frac{2}{m} \sum_{i=1}^m \epsilon_i \right\}\right]
\leq
\E_{\vec \epsilon}\left[\left| \frac{2}{m} \sum_{i=1}^m \epsilon_i \right|\,\right]
\leq
\frac{2}{\sqrt{m}},
\]
where we used Khinchin's inequality.
Finally, Khinchin's inequality also gives:
\[
\E_{\vec \epsilon}\left[\max\left\{0, \frac{2}{m} \sum_{i=1}^m \epsilon_i \right\}\right]
=
\frac{1}{2}\E_{\vec \epsilon}\left[\left| \frac{2}{m} \sum_{i=1}^m \epsilon_i \right|\,\right]
\geq
\frac{1}{\sqrt{2m}}.
\]
Next, let $F_m''$ contain ${m \choose m/2}$ functions, such that their projections on $\Z_m$ recover all the permutations of binary vector containing equal number of $0$ and $1$.
Clearly, in this case $\PRC_{m,n}(F_m'') = 1$.
Straightforward calculations show that in the same time $\Radc_m(F_m'') = 1 -2^{-m}{m \choose n}$ and we conclude the proof using upper and lower bounds on the binomial coefficient from \cite[Corollary 2.4]{S01}.
\end{proof}

\subsection{Proof of Theorem \ref{thm:risk-bound}}
\label{sect:proof-risk-bound}
The following version of McDiarmid's bounded difference inequality for the setting of sampling without replacement was presented in \cite[Lemma 2]{EP09} and further improved in \cite[Theorem 5]{CMP+09}:
\begin{theorem}[\cite{EP09}, \cite{CMP+09}]
\label{thm:McDiarmid}
Let $\Z_m$ be sampled uniformly without replacement from a fixed set $\Z_{m+u}\subseteq \Z$ of $m+u$ elements.
Let $g\colon \Z^m\to \R$ be a symmetric function s.t. for all $i=1,\dots,m$ and for all $z_1,\dots,z_m\in\Z$ and $z_1',\dots,z_m'\in\Z$,
\begin{equation}
\label{eq:bounded-difference}
\Bigl|
g(z_1,\dots,z_m) 
-
g(z_1,\dots,z_{i-1},z_i',z_{i+1},\dots,z_m)  
\Bigr|
\leq c.
\end{equation}
Then if $m=u$ with probability not less than $1-\delta$ the following holds:
\[
g \leq \E[g] + \sqrt{\frac{c^2N^3\log(1/\delta)}{8(N-1/2)^2}}.
\]
\end{theorem}
Note that function $\sup_{h\in \Hyp} \left( \err_h(\Z_u) - \err_h(\Z_m) \right)$ maps $(\X\times\Y)^m$ to $\R$ and is of course symmetric.
Straightforward calculations show that this function satisfies bounded difference condition \eqref{eq:bounded-difference} with $c=\frac{1}{m} + \frac{1}{u}$ (\cite[Inequality 9]{EP09}).
Theorem~\ref{thm:McDiarmid} states that with probability not less than $1-\delta$:
\begin{equation}
\label{eq:concentr-1}
\sup_{h\in \Hyp} \left( \err_u(h) - \err_m(h) \right)
\leq
\E_{\S_m}\left[\sup_{h\in \Hyp} \left( \err_u(h) - \err_m(h) \right)\right]
+
\sqrt{\frac{2N\log(1/\delta)}{(N-1/2)^2}}.
\end{equation}
Using upper bound of Theorem \ref{thm:tsym} with $L_{\Hyp}$ in place of  $F$ we complete the proof of~\eqref{eq:new-risk-bound}.
Next, consider a symmetric function $-\PRC_{m,n}(L_{\Hyp},\Z_m)$ which also maps $(\X\times\Y)^m$ to $\R$.
It can be shown again that it satisfies bounded difference condition \eqref{eq:bounded-difference} with $c=\frac{2}{m}$.
And thus, Theorem \ref{thm:McDiarmid} gives that with probability not less than $1-\delta$:
\begin{equation}
\label{eq:concentr-2}
\E_{\S_m}\left[\PRC_{m,n}(L_{\Hyp},\Z_m)\right]
\leq
\PRC_{m,n}(L_{\Hyp},\Z_m)
+
\sqrt{\frac{2N\log(1/\delta)}{(N-1/2)^2}}.
\end{equation}
Using this inequality together with \eqref{eq:new-risk-bound} in a union bound we obtain the second inequality of the theorem.

\section*{Appendix: Improving Lemma 3 of \cite{BM02}}
Let $\mu$ be a probability distribution on $\Z$ and $\X_m:=\{X_1,\dots,X_m\}$ be i.i.d. samples selected according to $\mu$.
Maximal discrepancy of $F$ was defined in \cite{BM02} as:
\[
\hat{D}_m(F, \X_m) = \sup_{f\in F} \left(
\frac{2}{m}\sum_{i=1}^{m/2} f(X_i)
-
\frac{2}{m}\sum_{i=m/2+1}^{m} f(X_i)
\right).
\]

It was shown in \cite{BM02} that if functions in $F$ are uniformly bounded by $1$ then:
\begin{equation}
\label{eq:bm02}
\frac{1}{2}\E\left[\Radc_m(F,\X_m)\right] - 2\sqrt{\frac{2}{m}}
\leq
\E\left[\hat{D}_m(F,\X_m)\right]
\leq
\E\left[\Radc_m(F,\X_m)\right] + 4\sqrt{\frac{2}{m}}.
\end{equation}
Since elements in $\X_m$ are i.i.d. the distribution of $\hat{D}_m$ is invariant under their permutations and thus 
$
\E\left[\hat{D}_m(F,\X_m)\right]
=
\E\left[\PRC_{m,m/2}(F,\X_m)\right].
$
Now we can use Theorem \ref{thm:dagstuhl-relation} to significantly improve bounds in \eqref{eq:bm02}:
\[
\E\left[\Radc_m(F,\X_m)\right] - \frac{2}{\sqrt{m}}
\leq
\E\left[\hat{D}_m(F,\X_m)\right]
\leq
\left(1 + \frac{2}{\sqrt{2\pi m} - 2}\right)\E\left[\Radc_m(F,\X_m)\right].
\]

\subsection*{Acknowledgments}
The authors are thankful to Marius Kloft and Ruth Urner for useful discussions and to the anonymous reviewers for their comments.
GB aknowledges support of the DFG through the FOR-1735 grant.
NZ was supported solely by the Russian Science Foundation grant (project 14-50-00150).

\end{document}